\documentclass{article}

\usepackage{PRIMEarxiv}

\usepackage[utf8]{inputenc} 
\usepackage[T1]{fontenc}    
\usepackage{hyperref}       
\usepackage{url}            
\usepackage{booktabs}       
\usepackage{amsfonts}       
\usepackage{nicefrac}       
\usepackage{microtype}      
\usepackage{lipsum}
\usepackage{fancyhdr}       
\usepackage{graphicx}       
\graphicspath{{media/}}     

\usepackage{amsmath,amsfonts}
\usepackage{algorithm}
\usepackage{array}
\usepackage[caption=false,font=normalsize,labelfont=sf,textfont=sf]{subfig}
\usepackage{textcomp}
\usepackage{stfloats}
\usepackage{url}
\usepackage{verbatim}
\usepackage{graphicx}
\usepackage{cite}
\usepackage[dvipsnames]{xcolor}

\usepackage{amsmath, amssymb, amsfonts, mathtools, amsthm, xspace, multirow, makecell} 
\makeatletter
\DeclareRobustCommand\onedot{\futurelet\@let@token\@onedot}
\def\@onedot{\ifx\@let@token.\else.\null\fi\xspace}

\def\eg{\emph{e.g}\onedot} 
\def\ie{\emph{i.e}\onedot} 
 
\def\etc{\emph{etc}\onedot}

\makeatother

\theoremstyle{plain}
\newtheorem{theorem}{Theorem}[section]
\newtheorem{proposition}[theorem]{Proposition}
\newtheorem{lemma}[theorem]{Lemma}
\newtheorem{corollary}[theorem]{Corollary}

\theoremstyle{definition}
\newtheorem{definition}[theorem]{Definition}

\theoremstyle{remark}

\providecommand{\R}{\mathbb{R}}
\providecommand{\p}{\mathcal{P}}
\title{FreeFlow: A Comprehensive Understanding on Diffusion Probabilistic Models via Optimal Transport
}

\author{
  Bowen~Sun, Shibao~Zheng \\
  Department of Electronic Engineering of SEIEE \\
  Shanghai Jiao Tong University \\
  Shanghai, China\\
  \texttt{\{sunbowen, sbzh\}@sjtu.edu.cn} \\
}

\begin{document}
\maketitle

\begin{abstract}
The blooming diffusion probabilistic models (DPMs) have garnered significant interest due to their impressive performance and the elegant inspiration they draw from physics. 
While earlier DPMs relied upon the Markovian assumption, recent methods based on differential equations have been rapidly applied to enhance the efficiency and capabilities of these models. 
However, a theoretical interpretation encapsulating these diverse algorithms is insufficient yet pressingly required to guide further development of DPMs. 
In response to this need, we present FreeFlow, a framework that provides a thorough explanation of the diffusion formula as time-dependent optimal transport, where the evolutionary pattern of probability density is given by the gradient flows of a functional defined in Wasserstein space. 
Crucially, our framework necessitates a unified description that not only clarifies the subtle mechanism of DPMs but also indicates the roots of some defects through creative involvement of Lagrangian and Eulerian views to understand the evolution of probability flow. 
We particularly demonstrate that the core equation of FreeFlow condenses all stochastic and deterministic DPMs into a single case, showcasing the expansibility of our method. 
Furthermore, the Riemannian geometry employed in our work has the potential to bridge broader subjects in mathematics, which enable the involvement of more profound tools for the establishment of more outstanding and generalized models in the future.
\end{abstract}

\section{Introduction}
Content generations by artificial intelligence are increasingly attractive because of their remarkable performance not only on image generation~\cite{goodfellow2020generative,karras2019style,rombach2022high} but also in broader domains such as context~\cite{devlin2018bert,openai2023gpt4} and video/audio generations~\cite{loeschcke2022text,zhang2022paddlespeech}.
DPMs inspired by diffusion phenomenon in physics~\cite{sohl2015deep} compose one of the most vibrant domains that have recently achieved considerable attention for its stable training and solid probabilistic deduction.
In a variety of scopes, DPMs show significant promise as a flexible approach to model complex high-dimensional distributions, whether on data generation or density estimation~\cite{kingma2021variational,haxholli2023faster}.
In terms of image generation, these methods take effect by learning from the simulated gradual diffusion of information then recover the origin inputs from noise through predicting the reverse process.
Earlier DDPM~\cite{ho2020denoising} applies Markovian assumption to conduct a series of diffusion steps, which is subsequently reformulated to non-Markovian process in DDIM~\cite{song2020denoising} and Ito process solved by the stochastic differential equation (SDE) in~\cite{song2020score}.
Similar to SMLD~\cite{song2019generative}, score matching is also used in~\cite{song2020score} for estimating data distribution during Langevin dynamics.
A series of ordinary differential equations (ODEs)~\cite{lu2022dpm,lu2022dpmpp,xu2022poisson,liu2023genphys} that exclude uncertainty by converting forward and backward process to deterministic procedure are then proposed to decrease calculation consumption.

\begin{figure}[t]
    \centering
    \includegraphics[width=0.85\columnwidth]{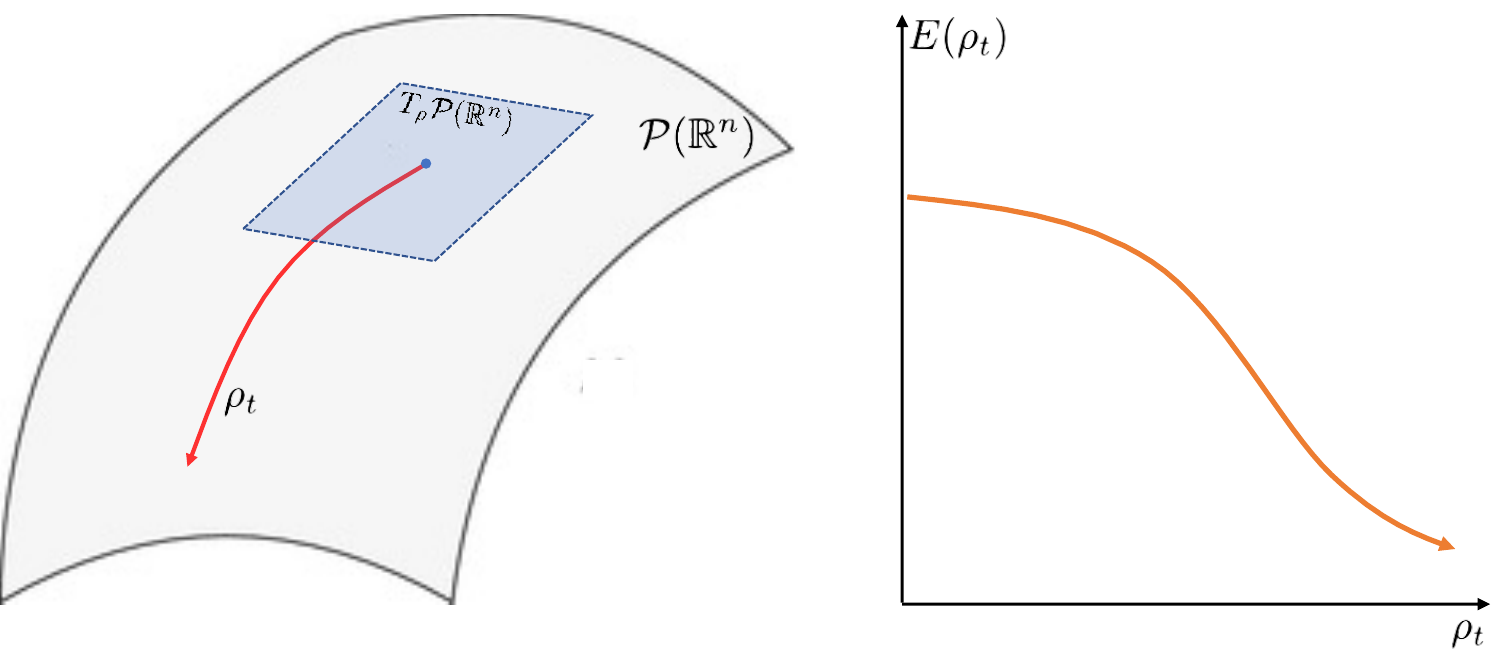}
    \caption{Illustration for the gradient flow of energy functional defined on Wasserstein space $\p(\R^n)$ that is an infinite dimensional Riemannian manifold comprising sets of probability measures on $\R^n$ endowed with Wasserstein distance.\protect\footnotemark[1] The red line with arrow and blue plane denote probability flow $\rho_t$ and tangent space $T_\rho \p(\R^n)$ respectively. The functional $E(\rho_t)$ shown on the right is decreasing most rapidly if $\rho_t$ evolves in accordance with its gradient flow.}
    \label{fig:illustration}
\end{figure}

The formulations appear to be divergent yet their ultimate fountain is worth investigation and unification.
\cite{luo2022understanding} suggests an uniform expression of these approaches by proving their equivalence to the hierarchical variational autoencoders, which is only limited to probably theory.
According to GenPhys~\cite{liu2023genphys}, it is possible for an extension to ODEs that satisfy continuity equations in physics to be considered as potential candidates for probability flow, which reflects the generality of DPMs from one perspective.
Unfortunately, there is still a lack of an expandable theoretical framework that not only can comprehensively explain the mechanisms of all DPMs but also serve as a bridge to involve more sophisticated analytical tools.

To address this requirement and explore the potential of DPMs, we present the FreeFlow declaring that the diffusion process is intrinsically the gradient flow of free energy functional defined on Wasserstein space, as illustrated in Fig.~\ref{fig:illustration}.
Combined with additional concepts from fluid dynamics, the FreeFlow further highlights that current models essentially compel random variables to stream in the manner of the Lagrangian description.
These perceptions presented in our work are consequently rooted in the optimal transport theory whose cost formula determine the metrics of Wasserstein space and tracks of probability evolution as well.
Primarily, we apply time-dependent optimal transport to observe the probability density as vector field $\rho_t$ on Riemannian manifold, then naturally transfer the distance between initial and final density to energy viewpoint by Benamou-Brenier theorem.
The gradient of functional $E(\rho_t)$ that represents the fastest dissipation direction is established as the kernel of our framework and afterward proved to be precisely the diffusion process of DPMs if a special formulation of $E$ is selected.

On basis of the FreeFlow, we are capable to undertake a more comprehensive investigation regarding DPMs. 
To be specific, we derives the Fokker-Planck equation by the gradient flow of a typical defined functional in FreeFlow thus constract its association with different kinds of diffusion algorithms.
In virtue of Lagrangian description, we indicate that linear map is the unique optimal transport with strictly convex cost function and analyze the danger of shock wave if randomly sampling data pairs.
It is also significant to emphasize the formulation of the cost function in optimal transport, particularly due to the eminent property of displacement interpolation we present when adopting a Euclidean distance as the cost.
By introducing extra analysis on propositions of the functional, which is useful for progressing of DPMs in the future.

\footnotetext[1]{Generally, Wasserstein space is defined by $\mathcal{P}(X)$, where $X$ is any compact metric space not necessarily the $\R^n$. Discussions are limited to $\mathcal{P}(\R^n)$ in this work.}

This paper is organized as follow: In Section~\ref{sec:background}, we introduce the overall background on DPMs about probability evolution in diffusion and related works; Section~\ref{sec:preliminaries} is devoted to fundamental concepts on optimal transport, gradient flows and Lagrangian-Eulerian descriptions; Section~\ref{sec:theoretical-framework} is the kernel of this paper which provides the definitions of FreeFlow, the deduction to Fokker-Planck equation and the displacement interpolation proposition for preparation; We subsequently apply FreeFlow to analyze DPMs through unified diffusion patterns and a Eulerian view for avoiding shock waves in Section~\ref{sec:rethink}; Conclusions and discussions are conducted in Section~\ref{sec:conclusion}.
Our main contributions are summarized as three points:
\begin{itemize}
    \item We propose an unified framework named \textbf{FreeFlow} to show that various diffusion patterns in DPMs can be intrinsically interpolated as the gradient flow of free energy functional, or equally, the geodesic in Wasserstein space.
    \item We demonstrate the Fokker-Planck equation is only one case of FreeFlow and analyze the significance impact on DPMs from formulations of cost function in optimal transport.
    \item FreeFlow highlights the Lagrangian description in fluid dynamics to observe current pipelines of DPMs, enabling the essential reveal on shock waves during straight generation and the reformulation to optimality equations by Eulerian manner.
\end{itemize}

\section{Background}\label{sec:background}
Within this section, we firstly present the explanation that pipelines of DPMs can be boiled down to the FP equation and then introduce some related works.

\subsection{Evolutionary Probability Density in DPMs}
Given an open subset $U$ of $\mathbb{R}^n$, the function $u(x, t):U \rightarrow \mathbb{R}$ of position $x$ and time $t$ subject to velocity field $v(x,t)$ and diffusivity $D(x,t)$ evolves, irrespective of sources or sinks, according to convection–diffusion equation:
\begin{equation}
\label{eq:convection–diffusion}
    \frac{\partial u(x,t)}{\partial t} \!=\! \nabla \cdot \bigl(u(x,t) v(x,t)\bigr) + \nabla \cdot \bigl(D(x,t)\nabla u(x,t) \bigr),
\end{equation}
where $\nabla \cdot$ and $\nabla$ denotes divergence and gradient with respect to position $x$ respectively.
It is essential for DPMs to be regarded as probabilistic version of Eq.~\eqref{eq:convection–diffusion} describing the diffusion process in physics, since it serves as the inspiration for diffusion-based frameworks.
Replacing $u$ with time-dependent probability density $\rho(x,t)$ (denoted as $\rho_t(x)$ for convenience), Eq.~\eqref{eq:convection–diffusion} is simply reformed to FP equation that is practically conducted in recent successful diffusion models.
If $D$ (denoted by $D_t$) is irrelevant of $x$, then Eq.~\eqref{eq:convection–diffusion} can be rewritten as an $n$-dimensional FP equation in the following form:
\begin{equation}\label{eq:fp}
    \frac{\partial \rho_t}{\partial t} = \nabla \cdot(\rho_t v_t) + D_t\Delta \rho_t,
\end{equation}
where $\Delta$ is Laplacian operator, and $\rho_t$ is probability density on $\R^n$ that is non negative and $\int_{\R^n} \rho(x,t) dx=1$.
For the sake of simplicity, independent variables will be omitted in subsequent discussions if there is no ambiguity, as done in Eq.~\eqref{eq:fp}.

Fokker-Planck equation describes the distribution evolution of probability density for random variables in Ito process.
Considering an $n$-dimensional random variable $X_t$ and the standard Wiener process with zero mean and unit variance $W_t$, the Ito process is given by
\begin{equation}\label{eq:ito}
    d X_t = \xi_t(X_t) dt + \sigma_t dW_t,
\end{equation}
where $\xi_t(X_t)=-v(x,t)$ and $\sigma_t=\sqrt{2D_t}$.
To be specific, $\xi_t(X_t)$ is the coefficient representing deterministic drift of the system influencing the mean shift of $X_t$, while $\sigma_t$ represents the variance of diffusion resulting from stochastic noise otherwise.
Under the constraint of sharing the same marginal probability densities at $t=0$ (the clean input) and the assumption of zero variance noise sampling, the practical reverse path of $X_t$ can be further rewritten to ODE:
\begin{equation}\label{eq:ODE}
    dX_t = f_t(X_t)dt,
\end{equation}
where the function $f_t(X_t)$ related to $\xi$ and $\sigma$ is the target for models to learn from in training stage.
Variations in the form of $f_t$ have a significant impact on the solution to Eq.~\eqref{eq:ODE}, leading to diverse performance outcomes for DPMs.
Besides, a stochastic process controlled by SDE is converted to a deterministic procedure, simplifying the realization of reverse.

The FP equation is a central component in the analysis of DPMs, as it not only provides crucial guidance during forward pipelines, but also has connections to more profound theories that are extensively discussed in Section~\ref{sec:example}.

\subsection{Related Works}
Based on the diffusion of probability flow with stochastic process in FP equation, different DPMs are proposed due to various formulation of $f_t(X_t)$.
The famous DDPM~\cite{ho2020denoising} realizes image generation under the assumption of Markovian process that is inherently discrete form of FP equation.
Endowed with Eq.~\eqref{eq:ito}, forward and reverse SDE are finally adopted in ~\cite{song2020score}.
ODEs can be derived by ignoring the noise item on the right hand side of Eq.~\eqref{eq:ito}, giving rise to equations expressing dynamics in continuous time along continuous path.
Without the stochastic item of SDE, the resultant ODEs makes it possible to boost the speed of generation in DPM-Solver~\cite{lu2022dpm} via introducing intensively developed ODE solvers.
With other selections of $f_t(X_t)$, some similar methods are proposed, \eg Poisson flow generative models (PFGM)~\cite{xu2022poisson} and GenPhys~\cite{liu2023genphys}.
Besides, the field of mathematics is currently abuzz with the study of optimal transport theory and gradient flows due to their intriguing associations.
We apply relevant results, \eg, energy functionals and displacement convexity on metric space to realize our analysis; see~\cite{mccann1997convexity,jordan1998variational,carrilloExampleDisplacementConvex2008,matthesFamilyNonlinearFourth2009,ambrosioCalculusHeatFlow2014}.

\section{Preliminaries}\label{sec:preliminaries}
In this section, we offer fundamental definitions and theories to lay the groundwork for our framework, which will be thoroughly analyzed afterwords.

\textbf{Optimal transport} theory has been continuously developed since Monge first presented this problem~\cite{monge1781memoire} and is currently in connection with Riemannian geometry, partial differential equations, gradient flow, \etc.
Given two probability spaces $(X, \mu)$, $(Y, \nu)$ and a cost function $c:X \times Y \rightarrow \mathbb{R}_{+} \cup \{+\infty \}$, the Monge problem is solving optimal map $T:X\rightarrow Y$ such that
\begin{equation}\label{eq:monge-problem}
    \inf \left\{ M(T) \coloneqq \left. \int_{X} c(x,T(x)) \mathrm{d} \mu \right| T_{\#} \mu =\nu \right\},
\end{equation}
where $T_{\#}\mu$ is push forward of $\mu$ subject to $(T_{\#}\mu)(A) \coloneqq \mu (T^{-1}(A))$ for any measurable set $A\subset Y$. 
Instead of finding the map $T$ in original Monge problem, the relaxed Kantorovich optimal scheme $K(\pi)$ is obtained by $\pi$ realizing
\begin{equation}
    \inf \left\{ K(\pi) \coloneqq \left. \int_{X \times Y} c(x,y) \mathrm{d} \pi(x,y) \right| \pi \in \Pi(\mu, \nu) \right\},
\end{equation}
where $\Pi(\mu, \nu)$ is the space composed of all joint probability measures 
$\pi$ on $X \times Y$ with marginals $\mu$ and $\nu$.
If $\mu$ and $\nu$ are two probability measures in Polish space $(\Omega,d)$, the Wasserstein metric with order $2$ between them is thus defined by 
\begin{equation}\label{eq:w-distance}
    W_2(\mu, \nu) \coloneqq \left( \inf_{\pi\in\Pi(\mu, \nu)} \int_{\Omega \times \Omega} d(x,y)^2 \mathrm{d}\pi(x,y) \right)^{1/2}.
\end{equation}
Equipped with distance $W_2$ as the metrics, the Wasserstein space $\p(\Omega)$ comprising all the set of probability measures on $\Omega$ is established.

\textbf{Gradient flows} are commonly used to describe certain equations in differential Riemannian space.
The Wasserstein space $\p(\Omega)$ is a classic example of an infinite-dimensional Riemannian space, from which one can derive the gradient flows.
For a time-dependent density function $\rho_t$ in Riemannian manifold $\p(\Omega)$ and functional $\Phi: \p(\Omega) \rightarrow \R$ assumed to be continuously differentiable, the gradient flow of $\Phi(\rho_t)$ on $\p(\Omega)$ is the equation
\begin{equation}\label{eq:gradient_flow}
    \frac{d\rho_t}{dt} = -\mathrm{grad}_{\rho_t} \Phi,
\end{equation}
where $\mathrm{grad}_{\rho_t}$ denotes the gradient of the functional at $\rho_t$.
Within this work, we typically consider the situation that $\Omega=\R^n$. 
Note that tangent space $T_\rho \p(\R^n)$ at $\rho$ is composed by functions $s$ on $\R^n$ that $\int s=0$.

\textbf{Lagrangian and Eulerian descriptions} are two perspectives for observing flow phenomenons connected by material derivative in the context of fluid dynamics.
The Lagrangian representation emphasizes trajectories of individual particles, whereas the Eulerian counterpart considers the physical quantity at fix positions in the field.
Likewise for probability flow field, the relation between Lagrangian and Eulerian view is given by 
\begin{equation}\label{eq:Lagrange-Euler}
    \left\{
    \begin{array}{l}
        \frac{d}{dt} \gamma_x(t) = v_t(\gamma_x(t)), \\
        \gamma_x(0) = x,
    \end{array}
    \right.
\end{equation}
where $\gamma_x(t)$ is the trace of particle $x$ at time $t$.
Their otherness and correlation are so universally applicable that we are permitted to treat FP equation in Eq.~\eqref{eq:fp} as Eulerian perspective yet the equal Ito process in Eq.~\eqref{eq:ito} as Lagrangian one. 
Moreover, if the velocity field $v_t(x)$ is Lipschitz continuous, there exists the unique solution $\gamma_x(t)$ to Eq.~\eqref{eq:Lagrange-Euler} for any initial point $x$ and $(t, x) \mapsto \gamma_{x}(t)$ is bijective and overall Lipschitz.
We will implement these properties for following analysis about defects of DPMs.

\section{Theoretical Framework}\label{sec:theoretical-framework}
We propose FreeFlow, whose form and related definitions are presented in Section~\ref{sec:freeflow}.
The Fokker-Planck equation is typically discussed as a special case of FreeFlow in Section~\ref{sec:example}.
Subsequent analysis on the convexity of cost functions in Section~\ref{sec:displacement-interpolation} highlights the importance of their formulations to prepare for further discussions about DPMs.

\subsection{FreeFlow}\label{sec:freeflow}
The FreeFlow framework provides a geometric interpretation of the diffusion-based evolutionary pipeline of probability density, which is formulated as a time-dependent optimal transport problem exploring the geodesic in Wasserstein space. 
Intrinsically, this diffusion process is linked to the gradient flow of the free energy function $E(\rho)$, which can be expressed as a differential equation:
\begin{equation}
    \frac{d}{dt} E(\rho) = -\left \langle \frac{\partial \rho}{\partial t}, \frac{\partial \rho}{\partial t} \right \rangle_\rho,
\end{equation}
where $\langle \cdot, \cdot \rangle_\rho$ denotes the Wasserstein scalar product of two vectors in tangent space $T_\rho\p(\R^n)$.
It can be further indicated by FreeFlow that DPMs are actually learning from the resultant direction of maximizing energy dissipation, which accounts for why multiple DPMs can operate.
To facilitate further investigation, we will first provide the definition of time-dependent optimal transport and then the metrics of space $\p(\R^n)$ taking the Riemannian geometric perspective of view.
\begin{definition}[Time-dependent Optimal Transport]
    If the continuous map/trajectory $\zeta_t(x)$ with $t\in [0,1]$ is associated by initial point $x$ and final point $y$ in space $\Omega$, where $\zeta_t(x)$ represents the displacement of $x$ at time $t$, then the time-dependent optimal transport map is
    \begin{equation}\label{eq:time-dependent-ot}
        \inf \left \{ \left. \int_X C(\zeta_t(x)) d \mu(x) \right | \zeta_0=\mathrm{id}, {\zeta_1}_{\#}\mu=\nu \right\},
    \end{equation}
    where $\nu(y)$ is the measure pushed forward from $\mu(x)$ and $C(\zeta_t(x))$ is the corresponding cost for displacement $\zeta_t(x)$.
    Moreover, time-dependent optimal transport is compatible with primal optimal transport if for all $x$ and $y$ we have
    \begin{equation}\label{eq:trajectory}
        c(x,y) = \inf \{ C(\zeta_t(x)) | \zeta_0=x, \zeta_1=y \}.
    \end{equation}
\end{definition}
By abuse of the notion, $\zeta_t$ tends to be a transport map in Eq.\eqref{eq:time-dependent-ot} similar to $T(x)$ of Eq.~\eqref{eq:monge-problem} and a trajectory in Eq.~\eqref{eq:trajectory}.
Note that $t \mapsto \zeta_t(x)$ should be at least segmental $C^1$ with respect to $t$ for $x$ of $\mu$-\ae thus velocity can be denoted by $\dot{\zeta_t}$.
While primal Monge problem solely pays attention on the initial and final positions, time-dependent OT calculates the cost by considering the traces of all particles involved.

\begin{definition}[Norm of $T_\rho\p(\R^n)$]
    If velocity field $v$ of particles evolving in accordance with probability density $\rho$ is completely controlled by their position, the norm $\| \cdot \|_\rho$ of tangent space $T_\rho\p$ is defined by
    \begin{equation}\label{eq:norm}
        \left\| \frac{\partial\rho}{\partial t} \right\|_\rho = \inf \left\{ \left. \int_{\R^n} \rho |v|^2 dx \right| \frac{\partial\rho}{\partial t} + \nabla \cdot (\rho v)=0 \right\}.
    \end{equation}
\end{definition}
Note that the probability density $\rho_t$ at time $t$ is the weak solution of continuity equation:
\begin{equation}\label{eq:continuity}
    \frac{\partial \rho_t}{\partial t} + \nabla \cdot (\rho_t v_t) = 0,
\end{equation}
which is shown as the condition of Eq.~\eqref{eq:norm}. 
The norm of $T_\rho\p(\R^n)$ is actually defined through borrowing the concept of total kinetic energy of particles in constraint of Eq.~\eqref{eq:continuity} that represents the conservation of mass.

\begin{definition}[Metrics of $\p(\R^n)$]
    Endowed with Eq.~\eqref{eq:w-distance} and Eq.~\eqref{eq:norm}, the Riemannian metrics of $\p(\R^n)$ can be given by $2$-Wasserstein distance:
    \begin{equation}
        W_2^2\!(\!\rho_0,\! \rho_1\!)\! = \! \inf \! \left\{\!\! \left.\int_0^1 \! \left\|\! \frac{\partial\rho}{\partial t} \!\right\|_{\rho(\!t\!)}^2 \!dt\! \right| \!\rho(0)\!=\!\rho_0,\! \rho(1)\!=\!\rho_1 \! \right\}\!,\!
    \end{equation}
    where $\rho_0$ and $\rho_1$ are two probability densities on $\R^n$ at time $t=0$ and $t=1$ respectively.
\end{definition}
We can regard this definition as the Benamou-Brenier problem minimizing the velocity field related action:
\begin{equation*}
    A(\rho, v) = \int_0^1 \left( \int_{\R^n} \rho_t(x) |v_t(x)|^2 dx\right)dt.
\end{equation*}
Thanks to the Benamou-Brenier theorem~\cite{benamou2000computational}, the square of Wasserstein distance is proved to be equivalent to the minimal action given by
\begin{equation*}
    W_2^2(\rho_0, \rho_1)=\inf \{ A(\rho,v)\}.
\end{equation*}
Therefore, the Wasserstein distance described by positions is converted to energy form where evolution under Eulerian view dominates.

\begin{definition}[Wasserstein Scalar Product]
    For two tangent vector $s_1$, $s_2$ in $T_\rho\p(\R^n)$, their Wasserstein scalar product is defined as
    \begin{equation}\label{eq:scalar_product}
        \langle s_1, s_2 \rangle_\rho = \int_{\R^n} \rho (\nabla \varphi_1 \cdot \nabla \varphi_2) dx, 
    \end{equation}
    where $s_i = \nabla \cdot (\rho \nabla \varphi_i)$ in $\R^n$.
\end{definition}
In reality, the Wasserstein scalar product is searching for the velocity field that minimizes total kinetic energy in Eq.~\eqref{eq:norm} satisfying the compatibility of continuity equation of Eq.~\eqref{eq:continuity}.
Thanks to the Benamou-Brenier formula, the velocity field $v_t$ is simply proved to be orthogonal with any solenoidal vector fields such that they are enabled to be the gradient field of some potential $\varphi$.
More details about why $v=\nabla \varphi$ is deduced in Appendix~\ref{apdx:metrics}.

Finally, the Wasserstein scalar product equips the definition of the gradient of a functional in the Wasserstein space, \ie, the Wasserstein gradient.
\begin{definition}[Wasserstein Gradient]
    For any smooth curve $\rho_t \in \p(\R^n)$, the Wasserstein gradient of functional $\Phi:\! \p(\R^n) \!\rightarrow\! \R$ at $\hat{\rho}$ is the unique function $\mathrm{grad}_{\hat{\rho}}\Phi$ such that
    \begin{equation}\label{eq:wasserstein_gradient}
        \left. \frac{d\Phi(\rho_t)}{dt} \right|_{t=0} = \left\langle \mathrm{grad}_{\hat{\rho}}\Phi, \left.\frac{\partial \rho_t}{\partial t}\right|_{t=0} \right \rangle_{\hat{\rho}},
    \end{equation}
    where $\rho_0=\hat{\rho}$.
\end{definition}
In consideration of Eq.~\eqref{eq:gradient_flow} and Eq.~\eqref{eq:scalar_product}, we attain the conclusion that the dissipation velocity of energy with respect to time is relevant to the Wasserstein gradient of $E$.
By selecting the form of the functional, multiple optimal directions are derived, giving rise to diverse generation trajectories proposed in generative diffusion methods.

\subsection{FreeFlow to Fokker-Planck Equation}\label{sec:example}
The FP equation, which is known to be the continuous expression equivalent to Markovian process in DDPM~\cite{ho2020denoising} and Ito process in SDE-based methods~\cite{song2020score}, is one of the fundamental principles of DPMs.
As a simple example of FreeFlow framework, one can naturally obtain the FP equation by deducing the gradient flow of an energy functional in specifically defined form.
\begin{theorem}\label{theorem:fp}
    Fokker-Planck equation is equivalent to the gradient flow of energy $E(\rho)$ given by
    \begin{equation}\label{eq:energy_defition}
        E(\rho) \coloneqq D_t \int_{\R^n} \rho \log \rho dx + \int_{\R^n} \rho \Psi dx,
    \end{equation}
    where $\Psi:\R^n \rightarrow \R$ is a smooth function subject to normalization condition that a constant $Z =\int_{\R^n} e^{-\Psi/D_t}$ exists.
\end{theorem}
\begin{proof}
    Taking $\frac{\delta \Phi(\hat{\rho})}{\delta \rho}$ as the first $L^2$-variation of $\Phi$, we have
    \begin{equation*}
        \left. \frac{d\Phi(\rho_t)}{dt} \right|_{t=0} = \int_{\R^n} \frac{\delta \Phi(\hat{\rho})}{\delta \rho} \left. \frac{\partial \rho_t}{\partial t} \right |_{t=0} dx .
    \end{equation*}
    Based on Eq.~\eqref{eq:wasserstein_gradient} and Eq.~\eqref{eq:continuity} with zero Neumann boundary condition, we then have
    \begin{equation*}
        \left\langle \mathrm{grad}_{\hat{\rho}}\Phi, \left.\frac{\partial \rho_t}{\partial t}\right|_{t=0} \right \rangle_{\hat{\rho}} = \int_{\R^n} \hat{\rho} \left ( \nabla \frac{\delta \Phi(\hat{\rho})}{\delta \rho} \cdot \nabla \varphi \right) dx .
    \end{equation*}
    Because of the definition of Wasserstein scalar product in Eq.~\eqref{eq:scalar_product}, a neat form of Wasserstein gradient is deduced to
    \begin{equation}\label{eq:grad_neat}
        \mathrm{grad}_{\hat{\rho}}\Phi = -\mathrm{div} \left[ \nabla\left( \frac{\delta \Phi(\hat{\rho})}{\delta \rho} \right) \hat{\rho} \right]. 
    \end{equation}
    Now we take $\Phi=E$ defined in Eq.~\eqref{eq:energy_defition} and substitute it to Eq.~\eqref{eq:grad_neat}, we obtain
    \begin{equation}\label{eq:fp_energy}
        \begin{aligned}
            \frac{\partial \hat{\rho}}{\partial t} =-\mathrm{grad}_{\hat{\rho}} E &= \mathrm{div} \left[ (D_t \nabla \log \hat{\rho} + \nabla \Psi)\hat{\rho} \right] \\
            &= \mathrm{div} (D_t \nabla\hat{\rho} + \hat{\rho}\nabla\Psi) \\
            &= D_t \Delta \hat{\rho} + \nabla \cdot (\hat{\rho}\nabla \Psi), \\
        \end{aligned}
    \end{equation}
    where the fisrt equation is from Eq.~\eqref{eq:gradient_flow}.
    The final Eq.~\eqref{eq:fp_energy} is obviously the Fokker-Planck equation shown in Eq.~\eqref{eq:fp} hence Theorem~\ref{theorem:fp} is proved.
\end{proof}
This theorem clearly demonstrates that the Fokker-Planck equation, or the forward diffusion dynamics in DPMs, is the steepest descent of free energy $E$ defined in Eq.~\eqref{eq:energy_defition}.
That is to say, the optimal evolutionary trace satisfying the gradient flow of the energy functional is implicitly adopted for models to learn from, which results in the efficiency of DPMs.
Specifically, the diffusion process in DPMs is achieved by the velocity field of $\rho_t$ moving towards maximal entropy and minimal energy. 
It is evident that the negative entropy, which is represented by $\int_{\R^n} \rho \log \rho dx$, needs to be minimized in accordance with the principle of maximum entropy. 
On the other hand, $\int_{\R^n} \rho \Psi dx$ acts as the energy term, with $\nabla \Psi$ affecting the deterministic component of the Ito process in Eq.~\eqref{eq:ito}. 

\begin{corollary}
    The energy functional defined in Eq.~\eqref{eq:energy_defition} is Kullback–Leibler (KL) divergence up to addition by a constant.
\end{corollary}
\begin{proof}
    As $t$ approaches infinity, $\rho_t$ in FP equation reaches a stationary state $\rho_\infty$ characterized by the Boltzmann distribution:
    \begin{equation*}
        \rho_\infty = \frac{e^{-\Psi/D_t}}{\int_{\R^n} e^{-\Psi/D_t} dx } = \frac{1}{Z}e^{-\Psi/D_t}.
    \end{equation*}
    The original energy functional can be hence rewritten as 
    \begin{equation}\label{eq:kl}
        E(\rho) = D_t \int_{\R^n} \rho \left( \log{\frac{\rho}{\rho_\infty}+\log{Z}} \right).
    \end{equation}
    If $Z=1$ and $D_t=1$, then Eq~\eqref{eq:kl} is the KL divergence of $\rho$ with respect to $\rho_\infty$.
\end{proof}
This corollary illustrates that when evolving along diffusion curves in DPMs, the gap of probability density is not measured by a strict mathematical distance but rather by KL divergence in practice.

\begin{proposition}\label{prop}
    If we specifically have $\Psi=\beta_t x^2/2$ where $\beta_t$ is a parameter irrelevant to $x$, then the stationary solution $\rho_\infty$ to Eq.~\eqref{eq:energy_defition} when $t$ approaches infinity is a normal distribution and $\rho_\infty \sim \mathcal{N}(0, \frac{D_t}{\beta_t}\mathbf{I})$.
\end{proposition}
\begin{proof}
    Substituting $\Psi=\beta_t x^2/2$ to the normalizing condition, we have
    \begin{equation*}
        Z = \int_{\R^n} e^{-\frac{\beta_t x^2}{2D_t}} dx = \sqrt{\frac{2\pi D_t}{\beta_t}}.
    \end{equation*}
    Therefore, the stationary solution is
    \begin{equation}
        \rho_\infty = \sqrt{\frac{\beta_t}{2\pi D_t}} e^{-\frac{\beta_t x^2}{2D_t}},
    \end{equation}
    which is a Gaussian distribution.
\end{proof}
This proposition presents that that with a subtle design of $\Psi$ and a sufficiently large value of $t$, the gradient flow of energy will eventually converge to a normal distribution.
Moreover, this property also enables models to be trained to revert from Gaussian white noise.
With perturbation gradually added to origin image $x_0$ in DPMs, the ultimate $x_t$ is hence close to a sample from this distribution controlled by $\beta_t$ and $D_t$.

\subsection{Displacement Interpolation}\label{sec:displacement-interpolation}
Although the explicit cost function appears to be discarded in the previous Eulerian field viewpoint, its impact on the transport process is inevitable due to its relation with displacement interpolation.
We demonstrate that transport functions deserve special consideration, as their structures can yield desirable properties that we prioritize.

Recalling transport cost function $C$ and $c$ in Eq.~\eqref{eq:trajectory}, we are permitted to rewrite their relation by
\begin{equation*}
    C(\zeta_t) = \int_0^1 c(\dot{\zeta_t}) dt,
\end{equation*}
if $\zeta_t$ is differential with respect to time and $c(\dot{\zeta_t})$ is differential transport cost.
We typically consider a special case that the cost function $c$ is convex on Euclidean space.
\begin{lemma}\label{lemma:covex-cost}
    Let $c$ be a convex function on $\R^n$, then
    \begin{equation}
        c(y-x) = \inf \left\{\left. \int_0^1 c(\dot{\zeta_t}) dt \right| \zeta_0=x, \zeta_1=y \right\}.
    \end{equation}
    By a stronger assumption that $c$ is strictly convex, then the unique minimal is obtained by the line:
    \begin{equation}
        \zeta_t = x+t(y-x), t\in[0,1].
    \end{equation}
\end{lemma}
This lemma can be proved by Jensen's inequality, which is also seen in Rectified Flow~\cite{liu2022flow} thus is hereby omitted.
\begin{definition}
    The function $\mathcal{X}:X \rightarrow \R$ is called to be $c$-concave if there exists a function $\mathcal{Y}:Y \rightarrow \R$ such that
    \begin{equation*}
        \mathcal{X} = \inf_{y\in Y} c(x,y) - \mathcal{Y}(y).
    \end{equation*}
\end{definition}
\begin{theorem}\label{theorem:displacement}
    If the transport cost function on $\R^n$ is strictly convex $c(x,y)=c(x-y)$ and $c(0)=0$, there is an unique $c$-concave function $\psi(x)$ presenting the solution to time-dependent optimal transport by
    \begin{equation}\label{eq:convex-zeta}
        \zeta_t(x)=x-t\nabla c^*(\nabla \psi(x)), t\in[0,1],
    \end{equation}
    where $c^*$ is Legendre transformation of $c$.
    Furthermore, if we specifically have $c(x,y)=|x-y|^2/2$ and $\mu$, $\nu$ are probability measures on $\R^n$, then there exists convex $\psi$ such that $\nabla \psi_{\#}\mu=\nu$ and the solution in Eq.~\eqref{eq:convex-zeta} is reformed to displacement interpolation:
    \begin{equation}\label{eq:displacement-interpolation}
        \rho_t=[(1-t)\mathrm{id}+t\nabla\psi]_{\#} \mu, t\in[0,1].
    \end{equation}
\end{theorem}
This theorem, which is fully proved in Appendix~\ref{apdx:displacement}, asserts that time-dependent optimal transport can be achieved through the linear interpolation of identical mapping and primal optimal transport mapping, provided that the cost function $c(x,y)$ is half of the square of Euclidean distance. 
In other words, optimality is realized throughout the entire transport process, resulting in $\zeta_t$ being the optimal transport from $\mu$ to ${\zeta_t}_{\#}\mu$.
By carefully formulating the implicit cost function $c$ according to Theorem~\ref{theorem:displacement}, we can naturally attain the optimal scheme through a linear transformation from the source to the destined distribution.

\section{Rethink DPMs by FreeFlow}\label{sec:rethink}
FreeFlow is valuable for its powerful theoretic viewpoint to elaborately explain benefits and reveal essential drawbacks behind DPMs. 
We demonstrate that our framework naturally encapsulates classic patterns of diffusion formulations in Section~\ref{sec:diffusion-pattern}. 
We then typically rethink straight line generation by revealing potential jeopardize of shock waves and deducing the optimality equation in Section~\ref{sec:shock-wave}. 

\begin{table}[t]
    \centering
    \begin{tabular}{c|c|c|c|c}
    \hline
    Type & Pattern & Methods  & $D_t$ & $\nabla\Psi(x)$\\
    \hline \hline
    \multirow{3}{*}{Stochastic} & Markovian & DDPM  & $\beta_t$ & $\beta_t x$ \\
    \cline{2-5}
    &\multirow{2}{*}{Ito Process}&VP-SDE& $\beta_t$      & $\beta_t x$  \\
    &                    &VE-SDE& $\dot{\alpha_t}$  & $0$   \\
    \hline
    \multirow{2}{*}{Deterministic}&\multirow{2}{*}{ODE}&DPM-Solver&$0$ &$f_t(x)$ \\
    &                 &PFGM& $0$               &$f_t(x)$\\
    \hline
    \end{tabular}
    \caption{The diffusion patterns in typical methods are represented by energy functional $E(\rho_t)$ of FreeFlow framework. 
    Corresponding values of parameters and formulation of functions are summarized in this table.\protect\footnotemark}
    \label{tab:forward_reverse}
\end{table}
\footnotetext{Specific forms of $\nabla\Psi$ for ODE pattern are included in Section~\ref{sec:diffusion-pattern} and hereby simplified to $f_t$ in the table due to complexity.}

\subsection{Diffusion Pattern}\label{sec:diffusion-pattern}
During the forward process, clean images undergo a gradual addition of stochastic noise until they are finally transformed into scheduled distribution.
Otherwise, the reverse process acting as the counterpart recovers noise to origin input by predictions from models.
FreeFlow can be used to summarize the diffusion patterns of mainstream DPMs as one of its applications.
The Theorem~\ref{theorem:fp} allows for various patterns to be consistently analyzed, with their parameters and formulations summarized in Tab.~\ref{tab:forward_reverse}.

\textbf{DDPM} takes the assumption of Markovian process that the conditional probability (or transition) $q(x_t|x_{t-1}) \!\sim\! \mathcal{N}(\sqrt{1-2\beta_t}x_{t-1}, 2\beta_t \mathbf{I})$ is a Gaussian distribution parameterized by $\beta_t$ for random variable $x_t$ and $x_{t-1}$, which leads to the final normal distribution of $q(x_t|x_0)$.
Note that $\beta_t$ hereby in perturbation kernel is the same parameter defined in Proposition~\ref{prop} impacting the velocity of drift.
In fact, such diffusion procedure from origin input $x_0$ to $x_t$ is equivalent to the gradient flow of free energy defined in Eq.~\eqref{eq:energy_defition}, because any continuous state Markovian process satisfies Chapman-Kolmogorov (CK) equation:
\begin{equation}
    q(x_{t_3}|x_{t_1}) = \int q(x_{t_3}|x_{t_2})q(x_{t_2}|x_{t_1}) dx_{t_2},
\end{equation}
where $t_3 > t_2 > t_1$.
CK equation is known to be directly derived to the FP equation which is proved in Section~\ref{sec:example} as an example of FreeFlow.
By definition of $D_t$ and $\nabla\Psi$ listed in Tab.~\ref{tab:forward_reverse}, we will obtain approximate standard normal distribution if images are perturbed by DDPM.

\textbf{SDE} methods (\eg, VE-SDE, VP-SDE) directly adopt Ito process as the diffusion procedure to produce perturbed data through Eq.~\eqref{eq:ito}.
While they differ from DDPM in inspiration and external form, the discrepancies can be limited to selection on the form of $D_t$ and $\Psi$.
As shown in Tab.~\ref{tab:forward_reverse}, VP-SDE is equivalent to DDPM with the drift item preserved that is able to be explained by Proposition~\ref{prop}, too.
Nevertheless, the drift item disappears in VE-SDE and $D_t$ varies by function $\alpha_t$ of $t$ making $q(x_t|x_{t-1}) \sim \mathcal{N}(x_{t-1}, \sqrt{2(\alpha_t - \alpha_{t-1})} \mathbf{I})$, where the variance increases solely.
SDEs uniformly evolves to normal distribution because of the entropy item, which is distinctly different from the following ODE methods.

\textbf{ODE} methods (\eg, DPM-Solver, PFGM, GenPhys) otherwise simplifies Ito process to a deterministic probability evolution via neglecting the diffusivity item, which makes it more likely to regard as flow methods~\cite{dinh2014nice}. 
Their noise scheduler may not necessarily proceed toward normal distribution consequently.
However, the simplified formulation, Eq.~\eqref{eq:ODE}, is obtained at the cost of more complex velocity field $f_t(x_t)$ conformed to continuity equation and should be invertible.
Their structures are thus not unique yet required to satisfy initial condition and possess straightforward final distribution for sampling backwards.
DPM-Solver rewrites Eq.~\eqref{eq:ito} to Eq.~\eqref{eq:ODE} and takes $\xi_t(X_t)-\sigma_t^2\nabla\log\rho_t(X_t)/2$ as $f_t(X_t)$.
In addition, PFGM involves Poisson equation and utilize Green's function to give $f_t$ by $f_t=\int \frac{(x-y)p(y)}{S_{n-1}(1)\|x-y\|^n} dy$, where $S_{n-1}(1)$ is the surface area of the unit $(n-1)$-sphere. 
GenPhy extends it to smooth PDEs whose solutions should behave as probability density.


The formulation of forward processes is varied, however, the flow directions of probability are consistently proved to be the Wasserstein gradient flow of energy $E$ proposed in Section~\ref{sec:example}.
In summary, FreeFlow generally explain the extensively adopted forward process as discrete variant consistently controlled by the gradient flow of $E(\rho)$.

\begin{figure}[t]
    \centering
    \includegraphics[width=0.85\columnwidth]{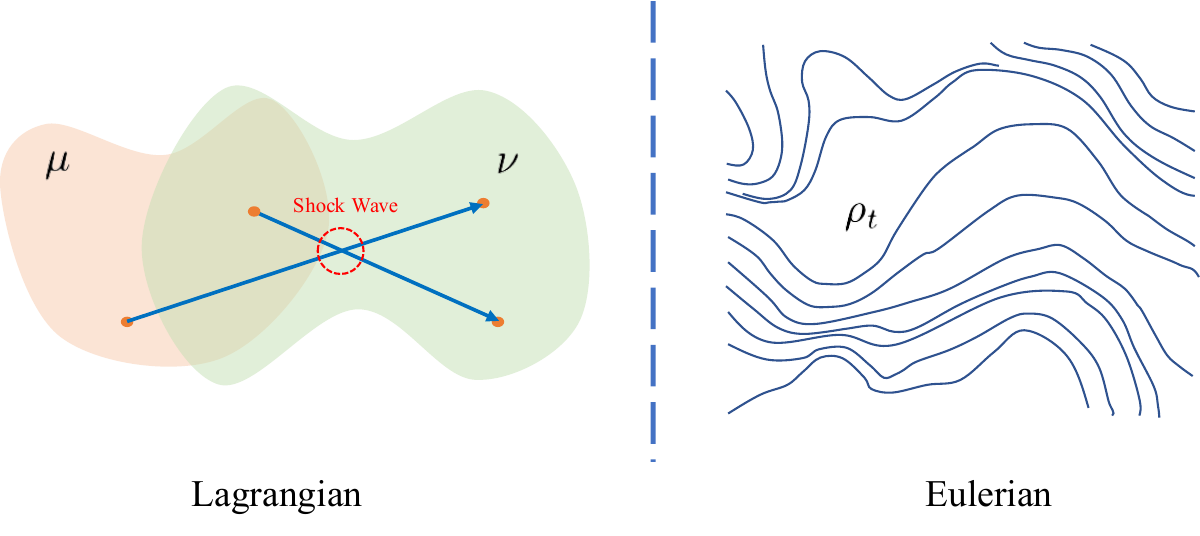}
    \caption{Illustration for Lagrangian (path lines, left) and Eulerian (streamlines, right) descriptions. A random transport from distribution $\mu$ to $\nu$ may trigger shock waves by the intersections as only individual particles are concerned in Lagrangian description; however, the probability density field $\rho_t$ at all positions are simultaneously presented in Euleraian manner without intersections.}
    \label{fig:shock-wave}
\end{figure}

\subsection{Shock Waves and Optimality Equation}\label{sec:shock-wave}
Generating images using DPMs can be a time-consuming process due to the multiple iterations of predictions during the reverse process. 
To eliminate this issue, one might anticipate a solution that simplifies the cumbersome gradual inversion into a single step. 
This possibility is explored in rectified flow~\cite{liu2022flow} where a linear transport track is effectively employed and a trick named reflow is additionally proposed to avoid the intersections of trajectories.
Nonetheless, flaws worthy of consideration are contained in rectified flow because of the disregarded cost function.

Indeed, shock waves illustrated in Fig.~\ref{fig:shock-wave}, as the essential form of the crossing, are possible to happen in finite time for compressible fluid even with smooth initial conditions when evolving by the Lagrangian method, which is the reason for its inevitable requirement for reflow. 
We have demonstrated in Lemma~\ref{lemma:covex-cost} that for a convex cost function, the optimal transport curve corresponds precisely to a straight line where particles move at a constant speed.
Transferring to Eulerian description, we reformulate it to a vector field with the lemma as follow.
\begin{lemma}\label{lemma:geodesic_euler}
    Let $v_0:\R^n \rightarrow \R^n$ be a differentiable vector field and $\zeta_t(x)=x-tv_0(x)$ is the trajectory field of particles with uniform motion, then the Eulerian velocity field $v_t$ associated with $\zeta_t$ satisfies
    \begin{equation}\label{eq:geodesic_euler}
        \frac{\partial v_t}{\partial t}+ v_t \cdot \nabla v_t=0.
    \end{equation}
\end{lemma}
\begin{proof}
    Since the velocity is a constant, the second derivative of $\zeta_t(x)$ is zero.
    We naturally have
    \begin{equation*}
        \frac{d^2}{dt^2} \zeta_t(x)=\frac{\partial v_t(\zeta_t(x))}{\partial t}+ v_t(\zeta_t(x)) \cdot \nabla v_t(\zeta_t(x))=0,
    \end{equation*}
    which is Eq.~\eqref{eq:geodesic_euler}.
\end{proof}
\begin{theorem}\label{theorem:characteristics}
    If the velocity field $v_t$ associated with $\zeta_t$ is consistently Lipschitz continuous and $\mu$ is the initial probability measure, then $\rho_t={\zeta_t}_{\#} \mu$ is the unique solution to 
    \begin{equation*}
        \frac{\partial \rho_t}{\partial t}+\nabla\cdot(\rho_t v_t)=0, \rho_0=\mu.
    \end{equation*}
    Combined with Lemma~\ref{lemma:geodesic_euler}, the time-dependent optimal transport in Eulerian view is given by
    \begin{equation}\label{eq:optimality-equation}
        \left\{
        \begin{array}{l}
            \frac{\partial \rho_t}{\partial t}+\nabla\cdot(\rho_t v_t)=0, \rho_0=\mu, \\
            \frac{\partial v_t}{\partial t}+ v_t \cdot \nabla v_t=0.
        \end{array}
        \right.
    \end{equation}
\end{theorem}
This theorem (proved in Appendix~\ref{apdx:characteristics}) implicitly authorizes the bond between cost function $c$ and velocity field $v_t$ thus gives rise to optimality equation Eq.~\eqref{eq:optimality-equation} from the Eulerian perspective.
Moreover, the formulation of $c$ itself, as demonstrated in Theorem~\ref{theorem:displacement}, fundamentally determines the trajectory and ensures the elimination of shock waves. 
Therefore, the relation of strictly convex $c$ and optimal initial velocity field $v_0$ can be given by
\begin{equation}
    v_0(x) = -\nabla c^*(\nabla \psi),
\end{equation}
where $c^*$ and $\psi$ are executed by the same definitions in Theorem~\ref{theorem:displacement}.
That is to say, the shock waves triggered by intersections can be radically shun if we implement proper cost function by designing relevant initial velocity field.

Moreover, the refrained shock wave is also guaranteed by the Eulerian field which is suitable to be declared as differential automorphism groups.
We recall Lagrange-Euler converter in Eq.~\eqref{eq:Lagrange-Euler} thus note that the trace $\gamma_x(t)$ is actually a homeomorphism $g_t(x) \coloneqq \gamma_x(t)$ deduced by the velocity field.
The optimal map $\zeta_t$ is obtained if given $\zeta_t\coloneqq g_t$ and the trace collection $g_t$ is a differential homeomorphism collection which is ensured by the regularity of Monge-Ampere equation with specific conditions in Theorem~\ref{theorem:displacement}.

\section{Conclusion}\label{sec:conclusion}
In this paper, we present the concept of FreeFlow, which is a framework that explores the comprehensive understanding of DPMs. 
The FreeFlow proposes that the diffusion process is inherently the gradient flow of a free energy functional defined on Wasserstein space, by which the Fokker-Planck equation is derived as the connection to various diffusion algorithms. 
With a unified expression in our work, the diffusion pattern crucially operated as the target for training models is summarized thus enables the possible development from more profound tools.
By combining ideas from fluid dynamics and time-dependent optimal transport theory, we highlight that current models make random variables move in a Lagrangian manner. 
Furthermore, the unique optimal transport with strictly convex cost function is identified as a linear map based on the Lagrangian description, and the potential danger of shock waves occurring when randomly sampling data pairs is analyzed. 
To eliminate the problem, the formulation of the cost function in optimal transport is emphasized, particularly when adopting a Euclidean distance as the cost. 
In summary, FreeFlow allows us to examine current designs about the diffusion process for better understanding and future development.

\bibliographystyle{unsrt}  
\bibliography{reference,refer-relate}
\newpage
\appendix
\section{Proof}\label{proof}
\subsection{Proof for $v=\nabla \varphi$}
\label{apdx:metrics}
\begin{proof}
    Suppose the kinetic energy in Eq.~\eqref{eq:norm} is minimized by $v_0$ and $\nabla\cdot w=0$, then for any $\varepsilon \ne 0$, $v_0+\varepsilon \frac{w}{\rho}$ is admissible.
    Due to the minimal energy, we have
    \begin{equation}
        \int \rho |v_0|^2 \leq \int \rho \left|v_0+\varepsilon\frac{w}{\rho}\right|^2,
    \end{equation}
    which should be always satisfied for any $\varepsilon$.
    This means that $\int v_0 \cdot w =0$, namely $v_0$ is orthogonal with any solenoidal vector fields $w$ thus there exists $\varphi$ that $v=\nabla\varphi$
\end{proof}
\subsection{Proof for Theorem~\ref{theorem:displacement}}\label{apdx:displacement}
\begin{proof}
    As the straight line is transport path of convex cost function, we thus consider $\zeta_1$ and prove $\zeta_1=\nabla c^*(\nabla \psi(x))$.
    For optimal transport scheme $\pi$, we have
    \begin{equation*}
        \psi(x)+\psi^c(y)=c(x,y),
    \end{equation*}
    where $(x,y)$ on the support set of $\pi$. We fix $(x_0,y_0)$ in support set and have
    \begin{equation*}
        \psi^c(y_0)=\inf_{x\in \Omega}c(x,y_0)-\psi(x).
    \end{equation*}
    The gradient will be zero at point $(x_0,y_0)$, namely $\nabla c(x_0,y_0)=\nabla \psi(x_0)$ and we obtain
    \begin{equation}\label{eq:interpolation}
        y_0 = x_0-(\nabla c)^{-1}(\nabla\psi(x_0)).
    \end{equation}
    Considering the convexity of $c$, the optimal transport is associated by Kantorovich potential $\psi$ thus given by $\zeta_1=x-\nabla c^*(\nabla \psi(x))$.
    Substituting $c(x,y)=|x-y|^2/2$ to Eq.~\eqref{eq:interpolation}, we have $\zeta_t=x-t\nabla \psi(x)$ and Eq.~\eqref{eq:displacement-interpolation} is obtained.
\end{proof}
\subsection{Proof for Theorem~\ref{theorem:characteristics}}\label{apdx:characteristics}
\begin{proof}
    We firstly prove that $\rho_t={\zeta_t}_{\#}\mu$ is the solution to the upper equation of Eq.~\eqref{eq:optimality-equation}.
    The map $t \mapsto \int \phi d \rho_t$ is Lipschitz for any smooth test function $\phi$ with compact support and its derivative with respect to $t$ is
    \begin{equation*}
    \begin{aligned}
        \frac{d}{dt} \int \phi d \rho_t &= \int \phi \frac{d\rho_t}{dt}\\
        &=-\int \phi d(\nabla\cdot(v_t\rho_t)) \\
        &=\int(\nabla\phi \cdot v_t) d\rho_t,
    \end{aligned}
    \end{equation*}
    where the second and third equation result from the continuity equation and the definition of divergence operator respectively.
    We rewrite $\int \phi d \rho_t$ by the definition of push forward measure and obtain
    \begin{equation*}
        \int \phi d \rho_t = \int \phi \circ \zeta_t d\mu.
    \end{equation*}
    Considering the compact support of $\phi$ and continuity of $\zeta_t$, $\phi \circ \zeta_t$ is Lipschitz so that give
    \begin{equation*}
        \frac{\partial (\phi \circ \zeta_t)}{\partial t}=(\nabla \phi \circ \zeta_t)\cdot \frac{\partial \zeta_t}{\partial t}=(\nabla\phi \circ \zeta_t)\cdot(v_t \circ \zeta_t).
    \end{equation*}
    Therefore, for $h>0$ we have
    \begin{equation*}
        \frac{1}{h} \left(\int \phi d \rho_{t+h} -\int \phi d \rho_t \right) = \int \left( \frac{\phi \circ \zeta_{t+h}-\phi \circ \zeta_t}{h} \right) d\mu,
    \end{equation*}
    where the integrand on the right hand converges to $\nabla(\phi \circ \zeta_t)\cdot v_t$ when $h \rightarrow 0$.
    We attain the conclusion that $t\mapsto\int\phi d\rho_t$ is differential for any $t$ and 
    \begin{equation}
        \frac{d}{dt} \int \phi d \rho_t = \int (\nabla\phi \circ \zeta_t)\cdot(v_t \circ \zeta_t) d\mu=\int \nabla \phi\cdot v_t d\rho_t,
    \end{equation}
    which proves that $\rho_t={\zeta_t}_{\#}\mu$ is the solution to the continuity equation.
    We then prove the uniqueness by supposing test function $\phi$ is Lipschitz with respect to $(x,t)$ with compact support and satisfies
    \begin{equation}
        \frac{\partial \phi}{\partial t}=-v\cdot\nabla\phi, \phi|_{t=T}=\phi_T.
    \end{equation}
    Similarly, $t\mapsto\int\phi d\rho_t$ is Lipschitz and for almost all $t$ we have
    \begin{equation*}
    \begin{aligned}
        \frac{d}{dt} \int \phi d \rho_t &= \int \phi_t \frac{d\rho_t}{dt} + \int \frac{\partial \phi_t}{\partial t} dt\\
        &=-\int \phi_t d(\nabla\cdot(v_t\rho_t))-\int (v_t\cdot\nabla\phi_t) d\rho_t=0.
    \end{aligned}
    \end{equation*}
    Therefore, we have 
    \begin{equation}\label{eq:rho_T}
        \int \phi_T d\rho_T=\int \phi_0 d\rho_0=0,
    \end{equation}
    which is kept for any $\phi_T$ such that $\rho_T=0$. 
    Thanks to the linearity of Eq.~\eqref{eq:optimality-equation}, the uniqueness can be given by proving that $\rho_0=0$ can derive $\rho_T=0$ for $\rho_t$ satisfying regularity conditions, which is presented in Eq.~\eqref{eq:rho_T}.
\end{proof}



\end{document}